\def\year{2018}\relax
\documentclass[letterpaper]{article} %DO NOT CHANGE THIS
\usepackage{aaai18}  %Required
\usepackage{times}  %Required
\usepackage{helvet}  %Required
\usepackage{courier}  %Required
\usepackage{url}  %Required
\usepackage{graphicx}  %Required
\usepackage{algpseudocode}
\usepackage[linesnumbered]{algorithm2e}
\SetKwInput{KwInput}{Input}
\usepackage{xcolor}
\usepackage{soul}
\usepackage[utf8]{inputenc}
\usepackage{tabularx} % set width of the table
\usepackage{booktabs}
\graphicspath{{img/}}
\usepackage{subcaption}
\usepackage{booktabs}
\usepackage{amssymb}
\usepackage{amsthm}
\usepackage{amsmath}
\usepackage{mathtools}
\usepackage{bm}
\newtheorem{Theorem}{Theorem}
\newtheorem{Lemma}{Lemma}

\newtheorem{Definition}{Definition}
\newtheorem{Rule}{Rule}
\newtheorem{Heuristic}{Heuristic}
\newcommand{\nlp}{\bm{v}}
\newcommand{\nlv}{\bm{v}}
\newcommand{\nav}{\bm{\omega}}
\newcommand{\nal}{\bm{\omega}}
\newcommand{\lv}{\bm{qv}}
\newcommand{\av}{\bm{q\omega}}
\newcommand{\fd}{\bm{qd}}
\newcommand{\fr}{\bm{qr}}
\newcommand{\force}{\bm{f}}
\newcommand{\qforce}{\bm{qf}}
\newcommand{\allsignvec}{\bm{\mathfrak{S}}}
\usepackage{scalerel}
\DeclareMathOperator*{\bigvecplus}{\scalerel*{\oplus}{\textstyle\sum}}
\usepackage{mathabx}

\newcommand{\vecdot}{\odot}
\newcommand{\veccross}{\otimes}
\newcommand{\signplus}{\oplus}
\newcommand{\signminus}{\ominus}
\newcommand{\signtimes}{\odot}
\newcommand{\tosign}{Q}

\begin{document}
\title{Towards Explainable Inference about Object Motion using Qualitative Reasoning}
\author {Xiaoyu Ge and Jochen Renz and Hua Hua \\ xiaoyu.ge@anu.edu.au
\\ Research School of Computer Science and Engineering
\\ The Australian National University
\\ Canberra, Australia}
%\author {Hua Hua \\ hua.hua@anu.edu.au}
%\author {Jochen Renz \\ jochen.renz@anu.edu.au }
\maketitle
\begin{abstract}
The capability of making explainable inferences regarding physical processes has long been desired. One fundamental physical process is object motion. Inferring what causes the motion of a group of objects can even be a challenging task for experts, e.g., in forensics science.
Most of the work in the literature relies on physics simulation to draw such inferences. The simulation requires a precise model of the underlying domain to work well and is essentially a black-box from which one can hardly obtain any useful explanation.

By contrast, qualitative reasoning methods have the advantage in making transparent inferences with ambiguous information, which makes it suitable for this task. However, there has been no suitable qualitative theory proposed for object motion in three-dimensional space. In this paper, we take this challenge and develop a qualitative theory for the motion of rigid objects. Based on this theory, we develop a reasoning method to solve a very interesting problem: Assuming there are several objects that were initially at rest and now have started to move. We want to infer what action causes the movement of these objects.
\end{abstract}

\section{Introduction}
% Why important
We are living in an era where an increasing number of AI agents entering into our daily lives and helping us with daily tasks such as household chores. To successfully perform these tasks, an AI agent needs to understand its surrounding environment and to be able to draw useful inferences based on their perceptual input. Living in a physical world requires AI be capable of inferring physical behaviours of everyday objects. This capability not only involves predicting what behaviours an object can have but also being able to figure out what causes their behaviours. In this paper, we focus on reasoning about object motion which is a most common physical behaviour of an object. Making an inference about object motion can be a challenging task for AI.

For example, Fig.~\ref{fig:before} shows a scene where a set of blocks were initially at rest. At a certain time, there was an action made to exert an impulse at one of the blocks, which caused the movement of the blocks as depicted in Fig.~\ref{fig:after}. When we observe this change, one natural question to ask is where and in which direction the impulse has been made. We humans can make such inference rapidly given only the information obtained from our visual perception. The knowledge we have about the scene is ambiguous, in the sense that we do not know exact physical parameters of the blocks, precise shapes or coordinates of their locations etc. However, we can still draw useful inferences based on this piece of knowledge and we can provide clear explanations of how the inference is derived. What human does in making the spatial or physical inference is conceptually similar \cite{hegarty2010components} to the methodology adopted by the qualitative reasoning community where the entities in that problem domain are characterised by a spatial representation, and the inference is drawn by reasoning about the constraints or relations between the entities.

\begin{figure}
\centering
  \subcaptionbox{\label{fig:before}}{\includegraphics[height=4cm,width=.48\linewidth]{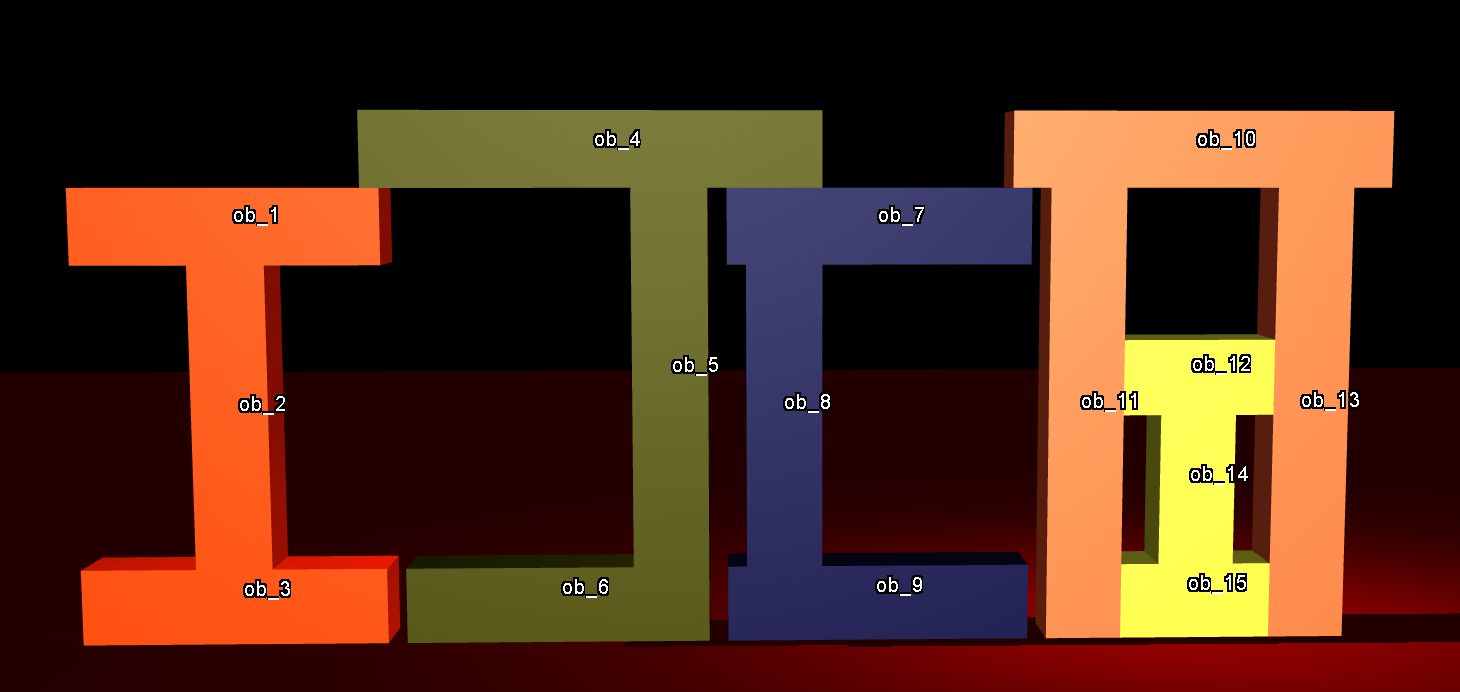}}
  \subcaptionbox{\label{fig:after}}{\includegraphics[height=4cm, width=.48\linewidth]{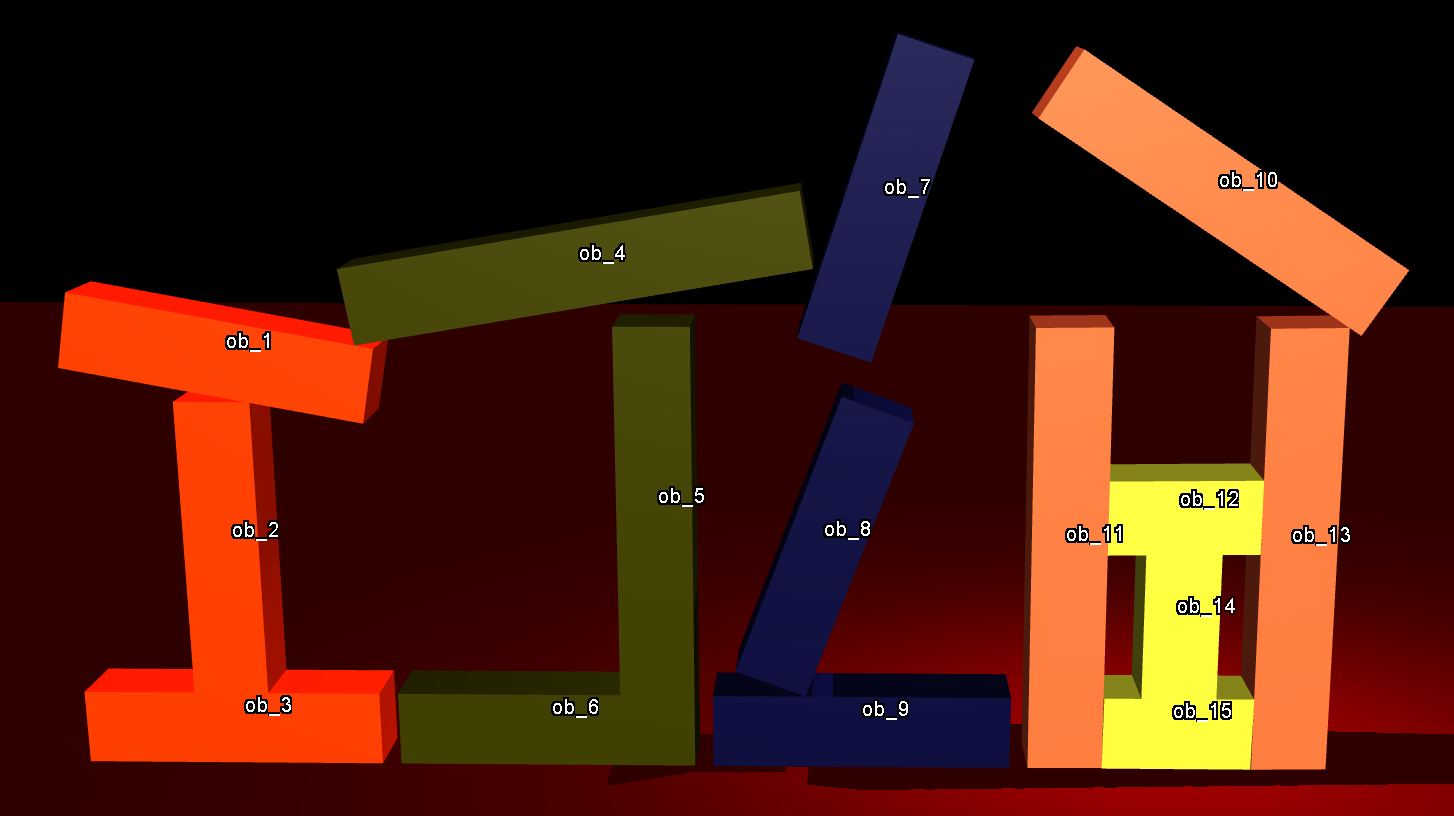}}
  \caption{\label{fig:example} An example scenario where we want to infer what action has been made to cause the illustrated movement of the objects. %\footnotesize
}

\end{figure}
As object motion in 3D space can be complex, it is critical to ensure that the qualitative representation is expressive enough and can capture all possible motions of an object in the space. Hence, we develop our theory according to a well-established physics modelling approach\cite{baraff1997introduction} that is also widely used in nowadays physics engines. We devise a qualitative representation for spatial entities and constraints that the modelling approach uses for motion prediction. We show that our theory can cover all the possibilities of the motion of objects in a system that can be described by the modelling approach. %Therefore, we would like to obtain all the possible and make them like real physical behaviours.
% In this paper we provide a solution that is based on the theory of both simulation and qualitative areas.
The key contribution of this paper is that we provide a qualitative theory for modelling rigid body motion in three-dimensional space. The theory is flexible as one can specify constraints in both qualitative or quantitative formulas based on their prior knowledge, and the reasoning method can be straightforwardly integrated into visual perception module. We demonstrate its usefulness by using it to solve a class of problems as illustrated in the above example.

%Understanding physical procedures is a key to understanding to , explaining and predicting. Predicting is more than explanation
%By observing the before and after scene, we want to infer what impulse has been made, and explain why.
%In this paper, we explored a way to. It is the first to use physical reasoning to achieve.
%Simulation also requires precise knowledge of the geometrical features of the objects, but we do not need.
%We start from this problem because it is similar to the capability we human use when we do forensics.
%This method is provided as a complementary method for physical reasoning.
\section{Background and Related Work}
There are two main research streams of modelling and reasoning about physical systems, namely qualitative physics and simulation-based reasoning.
Qualitative physics (for a survey, see \cite{davis2008physical}) emerged in the early 1980s, which uses symbolic approach to describe behaviours of physical systems. For example, \cite{forbus1984qualitative} proposed the qualitative process theory for modelling physical processes. \cite{forbus1981study} provided a reasoning method for the motion of a ball in 2D space. \cite{nielsen1990qualitative,stahovich2000qualitative} formalised different qualitative theories to describe constrained mechanical systems such as a clock. \cite{kuipers1990qualitative} developed a qualitative simulation framework that predicts physical behaviours based on qualitative differential equation models. However, this framework lacks a way to model object motion and forces in three-dimensional space. Physical interactions (e.g., surface contact) between extended objects are not considered in this framework either. Recently, there has been some qualitative spatial calculi developed for three-dimensional space such as the spatial representation of three-dimensional rotation \cite{asl2014qualitative} and trajectory \cite{mavridis2015qtc3d}. As far as we know, none of the above methods addresses motion of extended objects in three-dimensional space.

As an object movement can also be viewed as a spatial change of an object, the topic of our paper is broadly related to the area of qualitative spatial change and actions. One well-established work in this area is \cite{galton2000qualitative} that models qualitative state space based on the sign representation \cite{de1984qualitative}. A spatial change is then described as a trajectory through some state space. However, it does not deal with any problem related to reasoning about forces and their effects on object motion. For example, what would be the consequence of applying an action to a structure composed of rigid objects? Given a spatial change of a structure, what forces could lead to this change. In this paper, we will provide a solution to these problems.

In the domain of Angry Birds AI competition \cite{renz2016angry}, there has been some work on using qualitative reasoning \cite{walkega2016qualitative} or logic formalisation \cite{calimeri2016angry} to analyse behaviours of two-dimensional rigid objects in simulation. The rules are often empirically obtained and are specific to the problem domain. These methods also lack a formal investigation on to which degree the behaviours of the objects can be captured. By contrast, this paper developed a general formalism for a much more complex domain and established a connection between the formalism and a rigid body theory that is widely applied in physics simulations.

Simulation-based approaches \cite{kunze2015envisioning} have been widely used in robotics and cognitive science. When a system is completely modelled, the simulation can offer accurate predictions of the system behaviours. When there is an uncertainty in the model, it could be handled by probabilistic sampling \cite{Battaglia2013}. The main problem of simulation-based approaches is that they can hardly capture all possible behaviours of a physical system given partial observations. Besides, as simulation is essentially a numerical integration of equations based on the derivatives obtained from solving a complex constraint system, one can hardly derive any useful explanations out of it.

Recently, there has been work \cite{yilirim2017physical} on combining symbolic rules and geometric constraints to make physics inferences. \cite{tossaint2015logic} proposed a hierarchical framework where the high-level symbolic reasoning is performed to generate qualitative plans that will be instantiated by a geometric solver at the low level. This framework offers insights on combining qualitative and quantitative reasoning to solve challenging real-world problems, and our proposed theory can be naturally integrated into such hierarchical framework.

\section {A Qualitative Theory of Object Motion}
We consider the domain of rigid object dynamics and propose a qualitative theory for representing and reasoning about motion of rigid objects. The formalisation of the theory is inspired by works from the two different fields: simulation and qualitative reasoning. Specifically, we develop a qualitative representation to describe forces and their effects on object motion. The qualitative representation is based on sign calculus which has been proved as a versatile tool in modelling physical process \cite{de1984qualitative}. To model contact forces between objects and the physical constraints between the forces, we refer to the theory \cite{baraff1997introduction} of rigid body dynamics that is widely applied in many state-of-the-art physics engines. The goal of our theory is to capture all the possible motions of a group of rigid objects when the qualitative representation of their forces is known. Given an observed change in the motion of an object, the formalisation should also allow inferring what forces have caused the change.

In the below sections, we start by introducing a standard routine of rigid body simulation. We then present a qualitative representation and reasoning schema for object motion, which is in reminiscent of the simulation routine. Based on the proposed qualitative representation, we provide a formal definition of the action inference problem mentioned in the introduction.
\subsection{Rigid Body Dynamics in Simulation}
In a typical simulation of rigid body systems, the behaviour of the system is modelled as an ordinary differential equation that depends on time, given by $f(\bm{x}, t) = \dot{\bm{x}}$ where $\bm{x}$ is the state vector of the system and $\dot{\bm{x}}$ is the time derivative of the state. Time in the simulation is often discretised into time points. Given a system state $\bm{x_t}$ at time point $t$, to predict the state at a future time point $t + \Delta t$ the simulation runs a numerical method, e.g., using Euler's method the future state is calculated as

\begin{equation}\label{eq:euler}
\bm{x}_{t+\Delta t} = \bm{x}_{t} + \Delta t \cdot \dot{\bm{x}}_{t}
\end{equation}

The key step in the simulation is to calculate the time derivative $\dot{\bm{x}}$ at each time point. The simulation first performs collision detection to identify objects that are in contact with each other. The region of a contact area is approximated by a set of \emph{contact points} that are the corner points of the region. The simulation then computes normal and friction forces at the contact points according to certain physical constraints. Given that force is the time derivative of momentum, the momentum of the objects can be calculated based on the obtained forces according to Eq.~\ref{eq:euler}.

Hence, as long as we know what forces are acted upon an object at a given time point, we can predict the motion of the object at a next time point. On the other hand, as long as we know the difference between the object motion at two time points, we can infer what forces are contributing to these changes. Given that we do not know the precise parameters of the underlying system, the calculation above is likely to provide inaccurate results, and this inaccuracy will be accumulated during numerical integration, which makes it less likely to find real explanations. This problem can be solved by reformulating the rigid dynamics theory into a qualitative theory that has the advantage of dealing with ambiguous and imprecise knowledge.

\subsection {Qualitative Representation of Object Motion}
% The sign is given by the reference frame
%Table.~\ref{} shows a table of notations in this paper. We use lower letter $\bm{v}$ in bold to represent a vector of real numbers and $\bm{s}$
As force and momentum are vector quantities, we begin by introducing a standard qualitative representation for vectors. Specifically, we use sign calculus \cite{de1984qualitative} to represent a vector of numbers with each component of the vector is replaced by a sign that indicates whether the component is positive ($+$), negative($-$), or zero($0$). Hence, we make 27 distinctions of a three dimensional vector. We denote the set of all the distinct sign vectors as $\bm{\mathfrak{S}}= \{(a, b, c)| a,b,c \in \{+, - , 0 \} \}$. The inverse of a sign or sign vector $s$ is written as $s^{-1}$. E.g., $(+, -, 0)^{-1} = (-, +, 0)$. In this paper, we adopt a fixed reference frame with the $xy-$ plane representing the ground plane and the $z$ axis in the opposite direction of the gravity.

Fig.~\ref{fig:operation} shows the table of three basic arithmetic operations between signs, namely, addition($\signplus$), subtraction($\signminus$), multiplication($\signtimes$). The asterisk sign $*$ in the tables refers to an indefinite result with $* = \{+, -, 0\}$. The addition and subtraction between sign vectors are defined similarly, simply applying the corresponding sign operation between their components pair-wisely. A sign vector that has an indefinite component is treated as a set of sign vectors that has only definite components. E.g., the sign vector $(*, + , +)$ refers to a set $\{(0, +, +), (-, +, +), (+, +, +)\}$. We use a big addition symbol $\bigvecplus$ in the same way as the symbol $\Sigma$ used in the mathematical summation, which will generate a set of sign vectors as a result.
\begin{table}
					 \footnotesize
							 \begin{tabular}[t]{|c|ccc|}
									 \hline
									 $\signplus$ & $+$ & $0$ & $-$\\ \hline
										$+$& $+$ & $+$ & $*$ \\
									 	$-$& $*$ & $-$ & $-$ \\
									 	$0$& $+$ & $0$ & $-$ \\ \hline
							 \end{tabular}
							 \hfill
							 \begin{tabular}[t]{|c|ccc|}
									 \hline
									 $\signtimes$ & $+$ & $0$ & $-$\\ \hline
										$+$& $+$ & $0$ & $-$ \\
									 	$-$& $-$ & $0$ & $+$ \\
									 	$0$& $0$ & $0$ & $0$ \\ \hline
							 \end{tabular}
							 \hfill
							 \begin{tabular}[t]{|c|ccc|}
									 \hline
									 $\signminus$ & $+$ & $0$ & $-$\\ \hline
										$+$& $*$ & $+$ & $+$ \\
									 	$-$& $-$ & $-$ & $*$ \\
									 	$0$& $-$ & $0$ & $+$ \\ \hline
							 \end{tabular}
							 \caption{\label{fig:operation} The operation tables of sign addition, subtraction, and multiplication. The signs in the left columns are lefthand side operands.}
					 \end{table}
Another two fundamental operations of 3D vectors are inner product ($\cdot$) and outer product ($\times$), we define their sign-vector version in the same way as they are defined for numerical vectors:
\begin{equation}\footnotesize
	\begin{split}
	(u_1, u_2, u_3) \veccross (v_1, v_2, v_3) = ( (u_2 \signtimes v_3) \signminus  (u_3 \signtimes v_2), \\ (u_3 \signtimes v_1) \signminus  (u_1 \signtimes v_3), \\ (u_1 \signtimes v_2) \signminus  (u_2 \signtimes v_1)) \\
	(u_1, u_2, u_3) \vecdot (v_1, v_2, v_3) = (u_1 \signtimes v_1) \signplus (u_2 \signtimes v_2) \signplus (u_3 \signtimes v_3)
\end{split}
\end{equation}
\begin{figure}
  \centering
  \includegraphics[height=5cm]{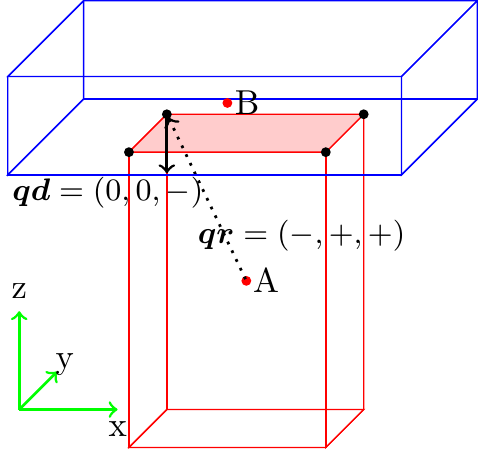}
    \caption{\label{fig:example_qualitative_force} Object A and B are in contact with the contact region highlighted in red. A contact force $\langle \fd, \fr, O_{A} \rangle$ is indicated by the black arrow. The red dots indicate the mass centres of A and B.}
\end{figure}
\begin{table}
\begin{tabular}{|l|p{5cm}|}
  \hline
  Notation  &  \\ \hline
  $O_i$, $O_{it}$, $\bm{O}_t$ & object $O_i$, state of $O_i$ at time $t$, states of set of objects.\\
  $\fd, \fr$ & sign vectors as illustrated in the left image. \\
  $\nlv, \lv$ & linear velocity and its sign vector representation. \\
  $\nav, \av$ & angular velocity and its sign vector representation. \\
  $\allsignvec$ & the set of all distinct sign vectors. \\
  $\tosign$ & a procedure that maps numerical entities to its sign representation. \\
  $\bm{D}$ & a set of qualitative forces. \\
  $\Delta(O_{t_1},O_{t_2})$ & a state change of $O$ from $t_1$ to $t_2$.\\
  $\bm{\Delta}_{\bm{D}}$ & all the possible state changes given by a set of qualitative forces $\bm{D}$. \\
  $\mathfrak{\bm{X}}_i$ & the force variables of object $O_i$. \\
  \hline
\end{tabular}
\caption{A list of notations used in this paper}
\end{table}

Based on this formalism, we now propose a qualitative representation for the forces in our domain.
\begin{Definition}[Qualitative Force]
A qualitative force on an object $O$ is a 3-tuple $\langle \fd, \fr, O \rangle$ where $\fd$ is a sign vector representing the qualitative direction of the force, $\fr$ is a sign vector of the direction pointed from the mass centre of $O$ to the point where the force is acted upon. The qualitative force of gravity on $O$ is $\langle \fd = (0, 0, -), \fr = (0, 0, 0), O \rangle$.
\end{Definition}
Given a qualitative force (see Fig.~\ref{fig:example_qualitative_force}), its components $\langle \fr, O \rangle$ refer to a qualitative location where the actual force is acted upon. We can obtain a more accurate region when the shape of an object is given. For notational convenience, we define a procedure $\tosign$ that will convert a 3D vector of real numbers to their sign counterparts or convert an actual force to its qualitative form.

We characterise the state of an object at a time point by the qualitative direction of its motion. % does the logic follow?
\begin{Definition}[Object State]
The state of an object $O$ at time $t$, denoted $O_t$, is a tuple $\langle \lv_t, \av_t \rangle$ where $\lv_t, \av_t$ are the sign vectors of the object's linear and angular velocity, respectively. There are 27 $\times$ 27 = 729 possible qualitative states.
\end{Definition}

\begin{Definition}[State Change] The state change $\Delta(O_{t_1},O_{t_2})$ of an object $O$ from $t_1$ to $t_2$ is defined as follows.
% \[ \Delta(O_{t_1},O_{t_2}) = \langle \lv_{t_2} \signminus \lv_{t_1}, \av_{t_2} \signminus \av_{t_1} \rangle \]
\[ \Delta(O_{t_1},O_{t_2}) = \langle \tosign(\nlv_{t_2} - \nlv_{t_1}), \tosign(\nav_{t_2} - \nav_{t_1}) \rangle \]
\end{Definition}
\begin{Definition}[Qualitative Action]
An action exerts a impulse at a point location $p$ on the exterior boundary of an object $O$. A qualitative action is a qualitative force representing the impulse force exerted by that action.
\end{Definition}
Now we formally define the problem we want to solve in this paper.
\begin{Definition}[Action Inference Problem]
An action inference problem AIP$\langle \bm{O}_{t_1}, \bm{O}_{t_2} \rangle$ is, given a set $\bm{O}$ of objects and their qualitative states $\bm{O_{t_1}}$ at time $t_1$ and a set of their qualitative states $\bm{O_{t_2}}$ at later time $t_2$, assuming there is an action made between $t_1$ and $t_2$, what is the qualitative representation of the action?
\end{Definition}
% we try to avoid abus eof notations, please refer to the list of notations.

\subsection {Reasoning about Object Motion}\label{sec:reasoning}
By Newton's second law of motion, the acceleration of an object is in the same direction of the net force on the object.
By reasoning about the differences between object velocities at two different time points, we could infer in which direction a net force can cause such change and we can further infer what forces on the object are contributing to forming the net force in the required direction.
% proposition, if x + y = c, then Q(c) \in Q(x) + Q(y)
% define non-negligent force
Specifically, given an object $O$, and let $\bm{D}$ be the set of qualitative forces. We can derive the set of all possible net forces on $O$ by enumerating all the possible combinations of the individual qualitative forces in $\bm{D}$, from which we can obtain all the possible state changes that can be led by these forces:

% NOTE if they are actual forces, then the result should be a direct sum. May be we need to talk about magnitude here?
\begin{equation} \label{eq:delta_d_o}
 \bm{\Delta}_{\bm{D}}
		 = \bigcup_{\bm{D_i} \in P(\bm{D})} \bigvecplus_{ \langle \fd, \fr, O\rangle \in \bm{D_i}} \fd \bigtimes \bigvecplus_{\langle \fd, \fr, O\rangle \in \bm{D_i}} \fr \veccross \fd
\end{equation}
where $P(\bm{D})$ refers to the power set of $\bm{D}$ and the symbol $\bigtimes$ refers to Cartesian product.

\begin{Lemma}\label{lemma:entail}
Given a state change $\Delta(O_{t_1},O_{t_2})$ resulted from a set of actual forces $\{\bm{f_1}, \bm{f_2}, \cdots \bm{f_n} \}$ acted upon $O$ between $t_1$ and $t_2$, let $\bm{QF}$ be a set of qualitative forces obtained by converting each actual force to its qualitative form, and let $\bm{D}$ be another set of qualitative forces. If $\bm{QF} \subset \bm{D}$, then $\Delta(O_{t_1},O_{t_2}) \in \bm{\Delta}_{\bm{D}}$.
\end{Lemma}
\begin{proof}[Proof]
Let $\nlp_{t1}$ be the linear velocity of an object at time $t_1$ and $m$ the mass of the object, let $\{\bm{f_1}, \bm{f_2}, \cdots \bm{f_n} \}$ be the set of the actual non-negligent forces on the object between time $t_1$ and $t_2$. By Newton's second law of motion, we have: \[\nlp_{t2} - \nlp_{t1} = \frac{1}{m}\sum_{i}^{n} \bm{F}_i, \bm{F}_i = \int_{t_i}^{t_j} \force_i dt, t_1 \leq t_i < t_j \leq t_2 \]
We can safely discard the constant $\frac{1}{m}$ and convert the equation to the sign representation. By definition of sign summation we have
\begin{equation}\label{eq:delta_v_numeric}
		\tosign({\nlp}_{t2} - {\nlp}_{t1}) = \tosign(\sum_{i}^{n} \bm{F}_i) \in \bigvecplus_{i} \tosign(\bm{F}_i)
\end{equation}
 By definition of sign addition, we have $\tosign(\bm{F}_i) = \tosign(\bm{f}_i)$ and replace it in Eq.~\ref{eq:delta_v_numeric} gives
\begin{equation}\label{eq:delta_v_sign} \tosign({\nlp}_{t2} - {\nlp}_{t1})
	\in \bigvecplus_{i}^{n} \tosign(\force_i)
	= \bigvecplus_{\langle \fd, \fr, O\rangle \in \bm{QF}} \fd \end{equation}
We can obtain the equation for the angular velocity in the similar way:
\begin{equation}\label{eq:delta_av_sign} \tosign(\nal_{t2} - \nal_{t1}) \in \bigvecplus_{\langle \fd, \fr, O \rangle \in \bm{QF}} \fr \veccross \fd \end{equation}
Because $\bm{QF} \subset \bm{D}$, we have $\bm{QF} \in P(\bm{D})$ and by Eq.~\ref{eq:delta_d_o}, we have $ \langle \tosign(\nlp_{t_2} - \nlp_{t_1}), \tosign(\nal_{t_2} - \nal_{t_1}) \rangle = \Delta(O_{t_1},O_{t_2}) \in \bm{\Delta}_{\bm{D}}$
\end{proof}
% direction set to <0, 0, 0> when the force is negligent
% In later section, we will further introduce combination method that work with external forces.
% negligent
Lemma.~\ref{lemma:entail} says that given the set of qualitative forces $\bm{D}$, as long as $\bm{D}$ contains all the qualitative forces of the actual non-negligent forces, one can always find the observed state change in $\bm{\Delta}_{\bm{D}}$ . This lemma is crucial for the search of qualitative forces that cause the state change as the lemma guarantees that the algorithm will never miss a candidate set that contains solutions.

Ideally, for every object, we want to obtain a small-sized $\bm{D}$ that contains all the actual qualitative forces on the object. The size of $\bm{D}$ is up to 27 $\times$ 27 = 729, which is the number of possible combinations of two sign vectors. Now we introduce several rules that can help to reduce the size of $\bm{D}$ without discarding any solutions. Each rule specifies a condition that has to be met by the assignment of the qualitative forces. We first specify the conditions in numerical formulas (can be used when precise numbers are known) and then define their qualitative versions by replacing the numerical operations with the sign operations.  Let $p$ be the contact point between two objects $O_i$ and $O_j$.
% Talk about different rules used for identifying forces.
% We just need to show Rule 1-4 does not prune geniue force directions

\begin{figure}
\centering
  \subcaptionbox{\label{fig:rule_vanishing_point_before}}{\includegraphics[height=3.2cm]{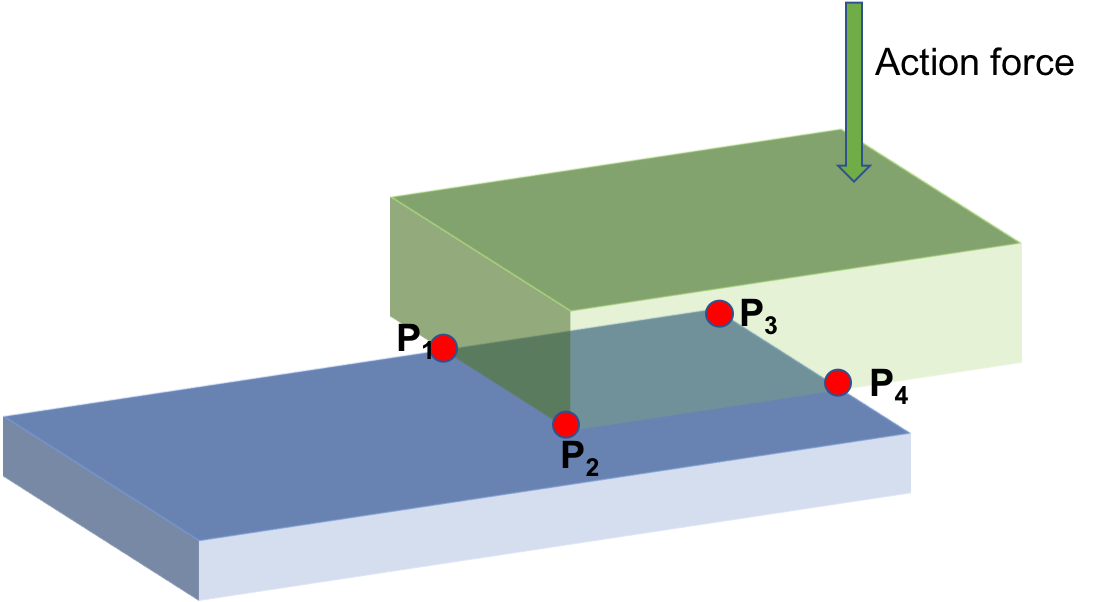}}
    \subcaptionbox{\label{fig:rule_vanishing_point_after}}{\includegraphics[height=3.2cm]{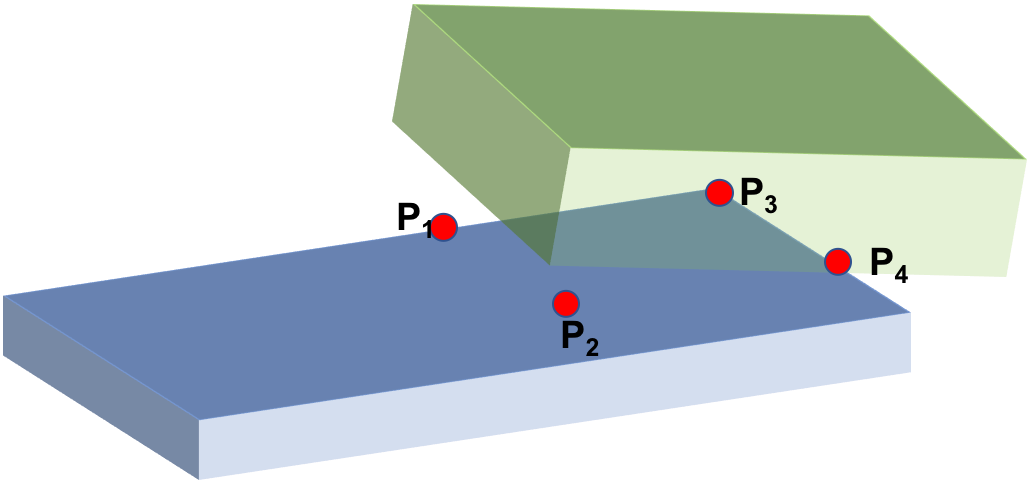}}
  \caption{\label{fig:rule_vanishing_point} An illustration of vanishing points. The green and blue objects are in contact initially (a), and the red dots indicate the contact points. After a suffficiently large force (indicated by the green arrow) is applied, the green object will start to move away at the points $p_1$ and $p_2$ which are considered as vanishing points.
  }
\end{figure}

\begin{Rule}[Vanishing Point]\label{rule:vanishing_point}
	A contact point $p$ is a vanishing point when $O_i$ and $O_j$ are moving away at $p$. There is no contact force at any vanishing point.
\end{Rule}
The concept of vanishing point is introduced in \cite{baraff1989analytical} based on the assumption that there is no attraction forces between objects. Therefore, when the two objects are moving away, the contact force will disappear at the point (see Fig.~\ref{fig:rule_vanishing_point}). The condition of not ``moving away" is given by \[\hat{n} \cdot (\bm{x}_i - \bm{x}_j) \leq 0, \bm{x} = \nlv + \nav \times r \]
where $\hat{n}$ is the contact normal and $\bm{x}$ is the linear velocity of the point $p$ on the object. The qualitative version of the condition can be defined in the same way:
% \[\tosign(n) \vecdot (\bm{qx}_i \signminus \bm{qx}_j) = +, \bm{qx}: \lv + \av \veccross \fr \]
\[\emptyset \neq \{-, 0\} \cap \bigcup_{\bm{q\delta} \in \bm{qx}_i \signminus \bm{qx}_j} \tosign(\hat{n}) \vecdot \bm{q\delta},	 \bm{qx}: \lv \signplus (\av \veccross \fr) \]

%\begin{Rule}[Dry Friction]\label{rule:dry_friction}
%	When $O_i$ or $O_j$ is moving, the friction force on $O_i$ at $p$ should be in the same direction of the vector $\bm{x}_j - \bm{x}_i$.
%\end{Rule}
%This rule is based on the definition of dry friction that the frictional force is always opposing relative movement between contact surfaces. The qualitative version is defined accordingly $\fd \in \bm{qx}_j \signminus \bm{qx}_i$.

\begin{Rule}[No Attraction Force]\label{rule:no_attraction_force}
	The direction $\hat{f}$ of the contact force on an object should be pointing inwards to the mass centre, given by $\hat{f} \cdot \hat{n} \geq 0$.
\end{Rule}
This rule is also based on the no-attraction-force assumption, which requires contact forces always push the two objects away from each other. The qualitative version requires $\{+, 0\} \cap \fd \vecdot \tosign(\hat{n}) \neq \emptyset $.

\begin{Rule}[Newton's Third Law of motion]\label{rule:third_law}
	Given a contact point between $O_i$ and $O_j$, and let $\fd_i$ and $\fd_j$ be the qualitative directions of the two contact forces on $O_i$ and $O_j$, respectively. $\fd_i^{-1}$ = $\fd_j$.
\end{Rule}
This rule is given by Newton's third law of motion: When two object are in contact, the contact force on one object should be in the opposite direction of the contact force on the other object.
\begin{Lemma}\label{lemma:rule}
	Given a rule that is satisfied by a set of actual forces $\{\bm{f_1}, \bm{f_2} \cdots \bm{f_n}\}$, the qualitative version of the rule is also satisfied by $\{\tosign({\bm{f_1}}), \tosign({\bm{f_2}}) \cdots \tosign({\bm{f_n}}) \}$.
\end{Lemma}
\begin{proof}[Proof]
The rule applies a series of arithmetic operations to the force vectors and is satisfied when the result is contained in the range specified by the rule condition. By the definition of sign calculus, each sign corresponds to a certain interval of real numbers, and the result of a sign operation covers all the possibilities of applying the corresponding numerical operation between numbers. Given that a qualitative rule is derived by replacing each numerical operation in the original rule with the corresponding qualitative operation, therefore, the qualitative rule will cover all the cases where the original rule is satisfied (the rule condition is met).
\end{proof}
\section{Solving Action Inference Problem}
\subsection{AIP-SAT}
% Explain why and how formalise the problem.
We solve \emph{AIP} by formalising it as a constraint satisfaction problem \emph{AIP-SAT} $\langle \bm{\mathfrak{X}}, \bm{\mathfrak{D}}, \mathfrak{\bm{C}} \rangle$ where $\bm{\mathfrak{X}}$ is the set of variables with each variable can be assigned with a value from its non-empty domain $\bm{\mathfrak{D_i}} \in \bm{\mathfrak{D}}$. $\bm{\mathfrak{C}}$ is set of constraints with each constraint specifies some relations that must be held between a subset of variables. The goal is to find an assignment of qualitative forces that can cause the observed change from $\bm{O_{t_1}}$ to $\bm{O_{t_2}}$.
\begin{Definition}[AIP-SAT]
	% maybe problematic, we should still use contact points.
Given an AIP problem $AIP\langle \bm{O_{t_1}}, \bm{O_{t_2}} \rangle$, and let $n$ be the number of force variables, we obtain the following AIP-SAT problem:
\begin{itemize}
\item $\bm{\mathfrak{X}} = \{x_{action}, x_1, x_2, \cdots x_n\}$: Each variable $x_{1:n}$ corresponds to a force at a contact point or the force of gravity at the mass centre, and $x_{action}$ is the variable of the action. Given an object $O_i$, $\bm{\mathfrak{X}}_i$ denotes a subset of variables whose forces are on $O_i$.

\item $\bm{\mathfrak{D}} = \{ \bm{\mathfrak{D}}_{action},  \bm{\mathfrak{D}}_1 \cdots \bm{\mathfrak{D}}_n \}$: Each domain $\bm{\mathfrak{D}}_{1:n} = \{\langle \fd, \fr, O \rangle: \fd\in \mathfrak{S} \}$ contains a set of qualitative forces that $x_{1:n}$ can be assigned with; $\fr$ and $O$ are fixed except for the action variable $\bm{\mathfrak{D}}_a = \{ \langle \fd, \fr, O \rangle: \fd, \fr \in \bm{\mathfrak{S}}, O \in \bm{O}\}$ as we need to infer the location upon which the action is exerted.
\item  $\bm{\mathfrak{C}}$: There are two constraints, namely,
\begin{itemize}
	\item $C_1: \forall O_i \in \bm{O}, \Delta(O_{it_1},O_{it_2}) \in \bm{\Delta}_{ \bm{D}_{ \bm{ \mathfrak{X} }_i }} $ where $\bm{D}_{ \bm{ \mathfrak{X} }_i }$ is a set of assigned values of the variables in $\bm{ \mathfrak{X} }_i$.
	\item $C_2: \forall x \in \bm{\mathfrak{X}}$, value of x is consistent with Rule~\ref{rule:vanishing_point}-Rule~\ref{rule:third_law}.
\end{itemize}
\end{itemize}
\end{Definition}

Constraint $C_1$ requires that for each object $O$, there exists a combination of forces that can change the qualitative state from $O_{t_1}$ to $O_{t_2}$.
Constraint $C_2$ can be viewed as a set of unary constraints that restrict the domain of the force directions so that the assignment does not break any physical rule introduced in Sec.\ref{sec:reasoning}.

\subsection {AIP-Solver}\label{sec:algorithm}
In this section, we introduce a complete AIP-SAT algorithm based on graph-based tree search.

\paragraph{Structure Graph} We use a directed multi-graph $G$ to represent a set $\bm{O}$ of objects and their spatial relations at time $t$. Each vertex $v$ refers to an object $O \in \bm{O}$ and is labeled by the state tuple $O_t$. The vertex also maintains a flag that can be set to a status of either ``checked" or ``to-check". There is an edge from  $v_i$ to $v_j$ if there is a contact point between $O_i$ and $O_j$. Each edge represents a variable. During the search, the algorithm will label an edge with a qualitative force tuple, which can be viewed as we assign a value to a variable.

The algorithm employs depth-first tree search to find solutions to a given AIP-SAT$\langle O_{t_1}, O_{t_2} \rangle$ problem. Each node in the search tree maintains a structure graph, and each vertex of the graph is labelled with the corresponding object state $O_{it_1}$ at time $t_1$. In the beginning, all the vertices are marked as ``to-check". When branching a node, the algorithm finds a partial assignment of the variables belonging to the object $O_i$ represented by the ``to-check" vertex. If a partial assignment is found, the algorithm will label the vertex with the corresponding object state $O_{it_2}$. After the branching, the flag of the vertex is set to ``checked". The algorithm keeps branching nodes until there is no more vertex ``to-check". An assignment is found when every vertex $v_i$ has been checked and labelled with $O_{it_2}$.

Let $\bm{D}_{kwn}$ be a set of variables of $O_i$ that have already be assigned values.

%To branch the root node, the algorithm identifies a set of objects that have different states at time $t_1$ and $t_2$, i.e., $\Delta(O_{it_1}, O_{it_2}) \neq \bm{0}$. For each identified object $O_i$, it obtains a set $\bm{QF}$ of qualitative forces that are complying with Rule.~\ref{rule:no_attraction_force}. It creates a new node by adding each qualitative force $\qforce \in \bm{QF}$ to the set $\bm{D}_{kwn}$ of $O_i$.

To branch the root node (see Algo.~\ref{algo:branch_root_node}), for each object $O_i \in \bm{O}$, the algorithm obtains a set $\bm{QF}$ of qualitative forces that are complying with Rule.~\ref{rule:no_attraction_force} (Line.~\ref{line:find_initial_forces}). It creates a new node by adding each qualitative force $\qforce \in \bm{QF}$ to the set $\bm{D}_{kwn}$ of $O_i$ (Line.~3-4). The algorithm will start searching first from a set of objects that have different states at time $t_1$ and $t_2$, i.e., $\Delta(O_{it_1}, O_{it_2}) \neq \bm{0}$.

To branch an intermediate node (see Alg.~\ref{algo:branch_intermediate_node}), the algorithm selects an arbitrary ``to-check" vertex $v_i$ in the graph, where preferences are given to the ones that have nonempty $\bm{D}_{kwn}$. %The algorithm considers the following two cases when creating new branches.
% Important that we will consider other contact points from the same object. This also prevent from circular updates.
%The first case refers to a scenario where $\Delta(O_{it_1}, O_{it_2}) = \bm{0}$. This scenario can happen when the mass of the object (e.g., non-movable objects) is so large that all the forces in $\bm{D}_{kwn}$ become negligent. In this case, the algorithm will just create a child node by labelling $v_{i}$ with $O_{it_2}$ in the graph. The second case refers to a scenario where
It then assign qualitative forces to variables so that $C_1$ and $C_2$ are satisfied. %$C_1$ is satisfied, i.e., $\Delta(O_{it_1}, O_{it_2}) \in \bm{\Delta}_{\bm{D}_{i}}$.
A backtrack happens when there is no such assignment.
% NOTE this condition seems "are not connected to any vertex that has already been checked" redunant

To make an assignment, the algorithm first obtains the set of incoming edges $\bm{E}_{in}$ that are not yet labeled and are not connected to any vertex that has already been checked, which provides a set $\mathfrak{X_i}$ of unassigned variables (Line.~\ref{line:get_incoming_edges_begin}-\ref{line:get_incoming_edges_end}). The next task is to find a valid assignment to those variables in $\mathfrak{X_i}$ so that the constraints hold with $\bm{D}_{kwn} \cup \bm{D}_{\bm{\mathfrak{X_i}}}$ (Line.~\ref{line:assignment}).

Having obtained the assignment, the algorithm adds the assigned variables to $\bm{D}_{kwn}$ (Line.~\ref{line:merge}), and then creates a child node by updating its graph in the following steps(Line.~\ref{line:update_begin}-\ref{line:update_end}):
\begin{itemize}
  \item Label $v_{i}$ with $O_{it_2}$ and set the flag to ``checked".
  \item Label the edges in $\bm{E}_{in}$ according to the assignment and label their corresponding outgoing edges $\bm{E}_{out}$ with the qualitative forces in the opposite direction by Rule.~\ref{rule:third_law}.
  \item For each vertex $v_j$ whose incoming edges labeled in the previous step, add the corresponding qualitative forces to the set $\bm{D}_{kwn}$ of $v_j$.
\end{itemize}
%1)  2)
%3)  %Update its object state according to the known qualitative forces and set the vertex as ``to-check".

% assumption, there are must be moving objects in $o_t$
Once a solution is detected, the assignment can be straightforwardly obtained from the set $\bm{D}_{kwn}$ of each vertex $v_i$.
\begin{algorithm}[t]
\KwInput{$\bm{O}:$ the set of objects in a scene}
  \For {$O_i \in \bm{O}$}
  {
    $\bm{QF}_i \leftarrow \{\langle \fd, \fr, O_i \rangle \textrm{ satisfies Rule.~\ref{rule:no_attraction_force}}, \fd,\fr \in \allsignvec \}$\; \label{line:find_initial_forces}

    \For{$\qforce \in \bm{QF}_i$}
    {
      $\bm{D}_{kwn} \leftarrow \{\qforce\}$\;
      Branch\_Intermmdiate\_Node($O_i, \bm{D}_{kwn}$)\;
    }
  }
 \caption{\label{algo:branch_root_node} Branch the root node}
\end{algorithm}

\begin{algorithm}[t]
\KwInput{$O_i, \bm{D}^{i}_{kwn}:$ the object at an intermediate node and the variables that have already been assigned.}
$\bm{E}_{in} \leftarrow \{e(v_j, v_i), \textrm{$v_j$ is labled with ``to-check"} \}$
\Comment $e(v_j, v_i)$: an edge from $v_j$ to $v_i$\; \label{line:get_incoming_edges_begin}
$\mathfrak{X_i} \leftarrow \textrm{variables given by }\bm{E}_{in}$\; \label{line:get_incoming_edges_end}
Assign each variable $x_i \in \mathfrak{X_i}$ so that  $C_1$ and $C_2$ is satisfied\; \label{line:assignment}
$\bm{D}^{i}_{kwn} \leftarrow \bm{D}^{i}_{kwn} \cup \bm{D}_{\bm{\mathfrak{X_i}}}$\; \label{line:merge}
Label $v_i$ with $O_{it_2}$ and set the flag to ``checked"\;
\For {$e(v_j, v_i) \in E_{in}$ \label{line:update_begin}} %\label{line:update_begin}
{
  label $e(v_j, v_i)$ with the assignment of the corresponding variable $x_i$\;
  label $e(v_i, v_j)$ with the qualitative force $\qforce$ given by Rule.~\ref{rule:third_law}. \;
  $\bm{D}^{j}_{kwn} \leftarrow \bm{D}^{j}_{kwn} \cup \{\qforce\}$\; \label{line:update_end}
}
 \caption{\label{algo:branch_intermediate_node} Branch an intermediate node}
\end{algorithm}
% If we do not know all the contact points
\begin{Theorem}[AIP-Solver is complete]
	Given an $AIP\langle O_{t_1}, O_{t_2} \rangle$ problem, the algorithm always finds an assignment that contains the actual qualitative forces leading to the state changes.
\end{Theorem}
\begin{proof}[Proof]
Let $\{\bm{f}_{action}, \bm{f}_1, \bm{f}_2 \cdots \bm{f}_n\}$ be the set of the actual forces that contribute to the state change and let $\bm{A}$ be the corresponding assignment with $\{x_{action} \leftarrow \tosign(\bm{f}_{action}), x_1 \leftarrow \tosign(\bm{f}_1) \dots x_n \leftarrow \tosign(\bm{f}_n) \}$.
We prove the theorem by contradiction. Assuming the assignment $\bm{A}$ is not found by the algorithm. Since the algorithm will check each vertex in the graph, the partial assignment of $\bm{A}$ must be pruned by the algorithm at a branching stage of the search.

Without loss of generality, assuming the partial assignment is pruned at a node with a ``to-check" vertex $v_i$, and let $\bm{D}_{\bm{A}_i}$ be the set of qualitative forces on $O_i$ according to the assignment $\bm{A}$. Given the fact that the node is pruned, by the construction of the algorithm, we know $\Delta(O_{it_1}, O_{it_2}) \notin \bm{\Delta}_{\bm{D}_{i}}$.
When $\bm{D}_{\bm{A}_i} \subset \bm{D}_i$, the pruning contradicts Lemma 1.
% Note, be careful currently we do not allow forces of the same source be adde to D_i again. It may make the following statement false.
When $\bm{D}_{\bm{A}_i} \not\subset \bm{D}_{i}$, it means that some of the qualitative forces in $\bm{D}_{\bm{A}_i}$ have been discarded by the algorithm. As the algorithm only discards assignments if they violate Rule~\ref{rule:vanishing_point}-\ref{rule:third_law}, given the actual forces do not violate any rule, the pruning contradicts Lemma 2.
\end{proof}

The branching factor of each intermediate node is equal to the number of possible partial assignments of the variables in $\mathfrak{X_i}$. This number can be huge when there are many variables that can be assigned to multiple values. To avoid the expensive branch factor, in practice, we can group some assignments of a variable into a set $A_i$ and the set is considered as a single assignment. A constraint is satisfied by $A_i$ as long as it is satisfied by any individual assignment in $A_i$. By this, we can substantially reduce the branching factor while it can be proven that the completeness of the algorithm is still guaranteed.

The algorithm can be made even more efficient by using heuristics that capture certain knowledge of the underlying domain. Here we give two example heuristics:

Given a variable of a contact force and its assigned qualitative force, if the assignment is not made by Rule.~\ref{rule:third_law}, we call this qualitative force \emph{resistant force}.
\begin{Heuristic}
If a qualitative force is a resistant force, it can only cancel, but not overwhelm, the effects of other forces.
\end{Heuristic}
For example, given a resistant force of direction $\langle +, -, 0 \rangle$ and another non-resistant force $\langle -, -, 0 \rangle$. Combining the two forces by this heuristic gives $\langle \{-, 0\}, -, 0\rangle$. We create this heuristic according to the standard simulation routine \cite{baraff1997introduction2} of computing contact forces. Once the simulation detects a potential inner-penetration (collision) between objects, it will calculate a contact force (resistant force) to resist the inner-penetration. When the collision is inelastic, the effect of this resistant force on object's momentum can hardly overwhelm the effect of the force that causes the inner-penetration.
\begin{Heuristic}
The action that is made to an object will cause the movement of the object.
\end{Heuristic}
This heuristic assumes that an object will change its state after an action is made to it. Under this assumption, the algorithm can start the search from only the objects whose states have changed.

% Skip talking about how to find an assignment here. As need a further experiments.
 % Bug? after a long time, object will always be static
% requires we know all the intermediate contact points as well.
% The rest of edges are assigned with zero.
% The more direction of forces we can infer from the observation, the more we can prune.
%When there are multiple collisions, we have to infer possible places that they can be colliding with, it requires to know the geometry of the objects, which is left to be done as a future work.
%There could be second collision happening. If there is a second collision, we need to take the geometry into account.

% Infer missing contact point?
% NOTE add more forces does not break the rule!!! therefore, we can add the normal force direction
% Completeness, if there is a solution, we must find the algorithm.
%\Theorem {Completeness of the Algorithm}
%If there exists a forces that can make the observed changes, the algorithm can infer it.
%\proof
%Sketch
%\endproof
%The number of explanations decreases as the more information we can obtain from the prior knowledge.
%Show if we know more, we can be more precise
\section{Evaluation}
We implemented the algorithm and evaluated it in both simulated and the real-world environments.
%\paragraph{Simulation}

We use Mujoco(www.mujoco.org) that is a state-of-the-art physics engine used in the robotics research. We create several scenes with each scene contains a set of blocks of different sizes and physical properties. In the beginning the blocks are forming a stable structure, and we use the state of the objects as the \emph{before state}. Since we use simulation, we obtain the locations of contact points and mass centres directly from the simulator, and use them to compute the sign vectors. One thing to emphasise is that we do not have to know the exact numbers as it does not make any difference to the result if two different numbers give the same sign vector. We make an action that exerts an impulse on one of the objects and the impulse is chosen to cause a significant movement of objects. We simulate the scene for a number of time steps and then take the state of the objects as the \emph{after state}. We run the algorithm on the two states to infer the qualitative representation of the action. As we proved, the algorithm found correct actions in all the experiments.
\begin{figure}
\centering
\includegraphics[width=\linewidth]{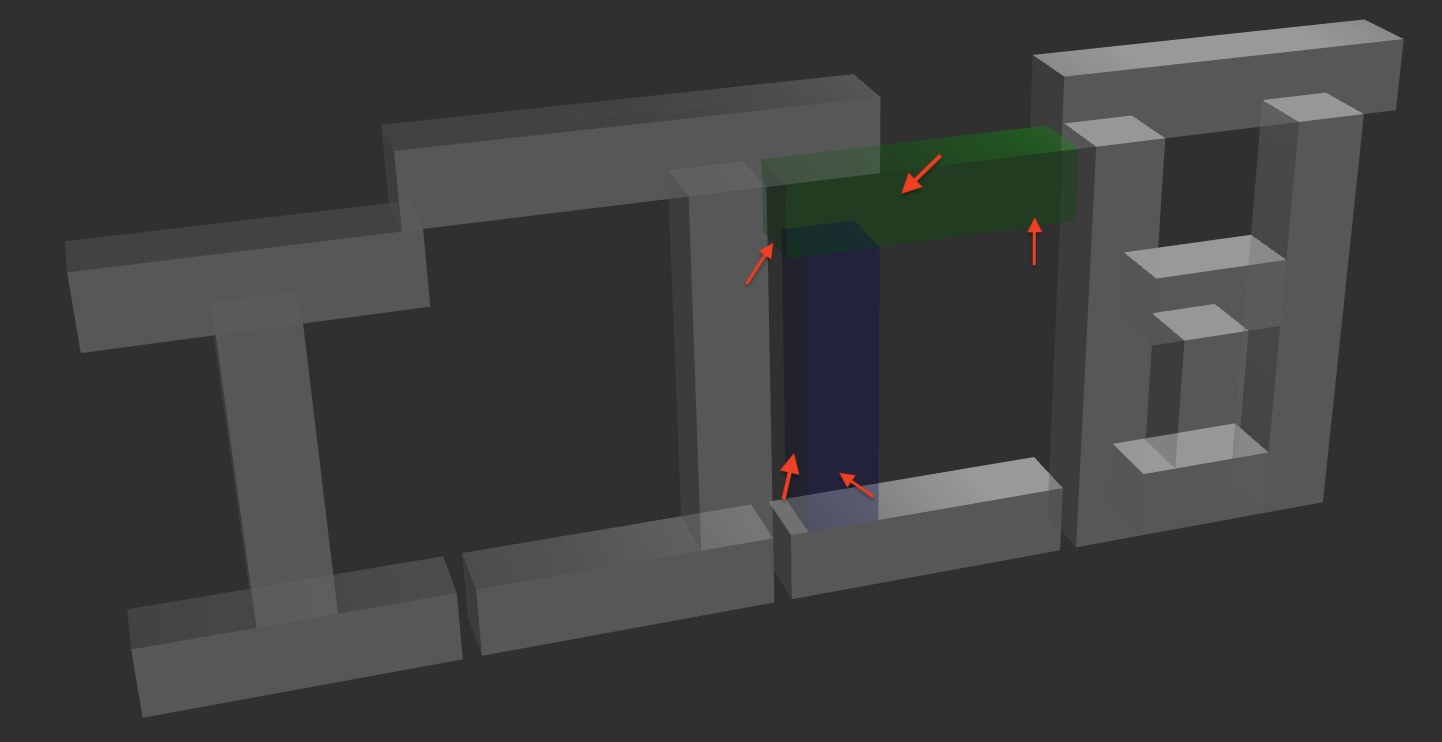}
  \caption{\label{fig:example_result}The method found multiple qualitative forces (some are depicted in red arrows) on the blue and green block as solutions to the example problem.}
\end{figure}

Given a scene, the total number of candidate qualitative actions is equal to the number of all the possible combinations of qualitative force directions (except the zero sign vector) and their possible qualitative locations . In the example given in Fig.~\ref{fig:example} where there are 15 objects, there are 26 $\times$ 27 $\times$ 15 = 10530 possible qualitative actions. The algorithm finds 48 different qualitative forces on two objects using Heuristic~1-2 while using only Heuristic~2 it finds 772 different qualitative forces on the same objects (see Fig.~\ref{fig:example_result}).
\begin{figure}\label{fig:simple_example}
\centering
  \subcaptionbox{\label{fig:simple_example_before}}{\includegraphics[height=4cm, width=.48\linewidth]{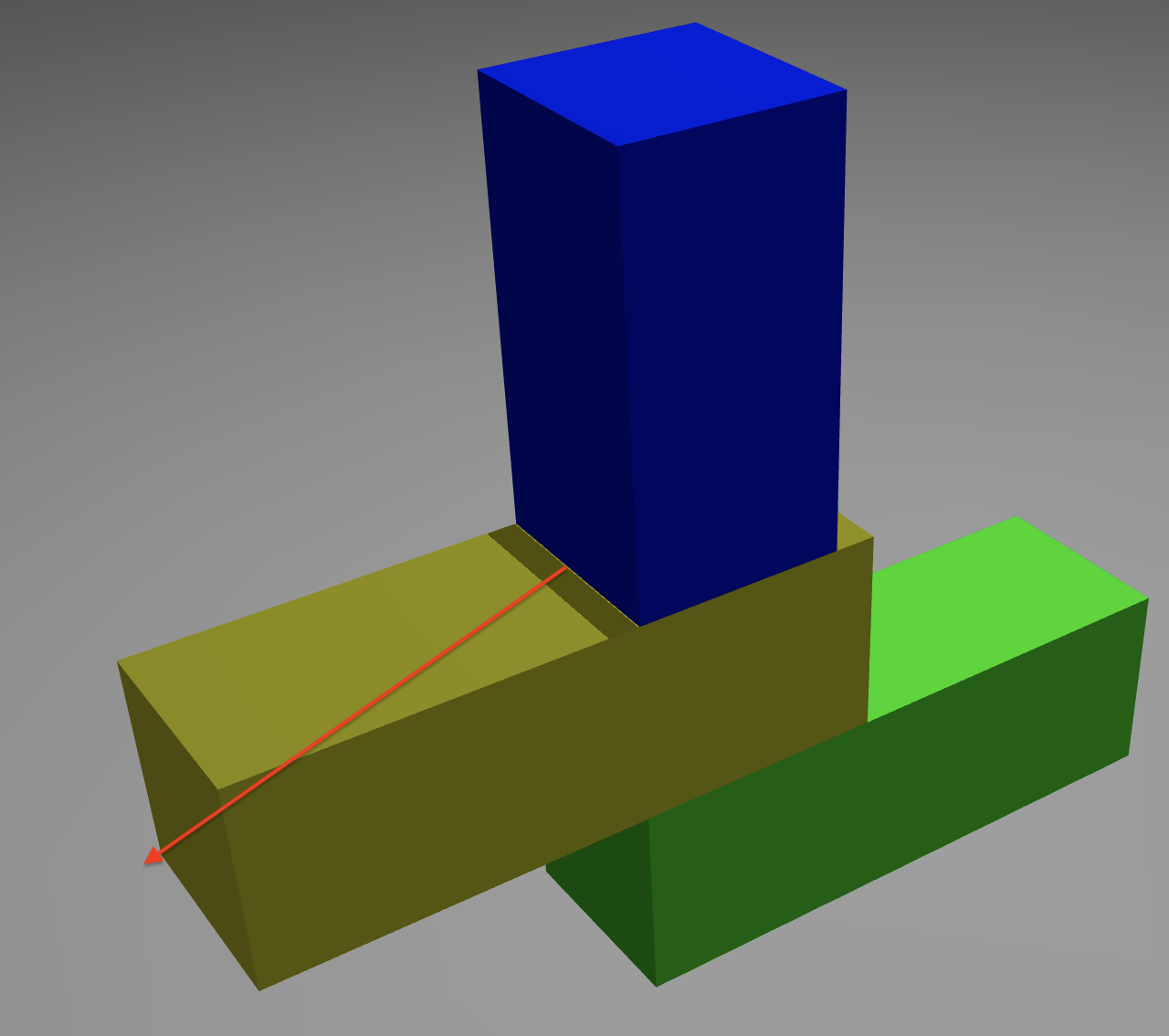}}
    \subcaptionbox{\label{fig:simple_example_after}}{\includegraphics[height=4cm, width=.48\linewidth]{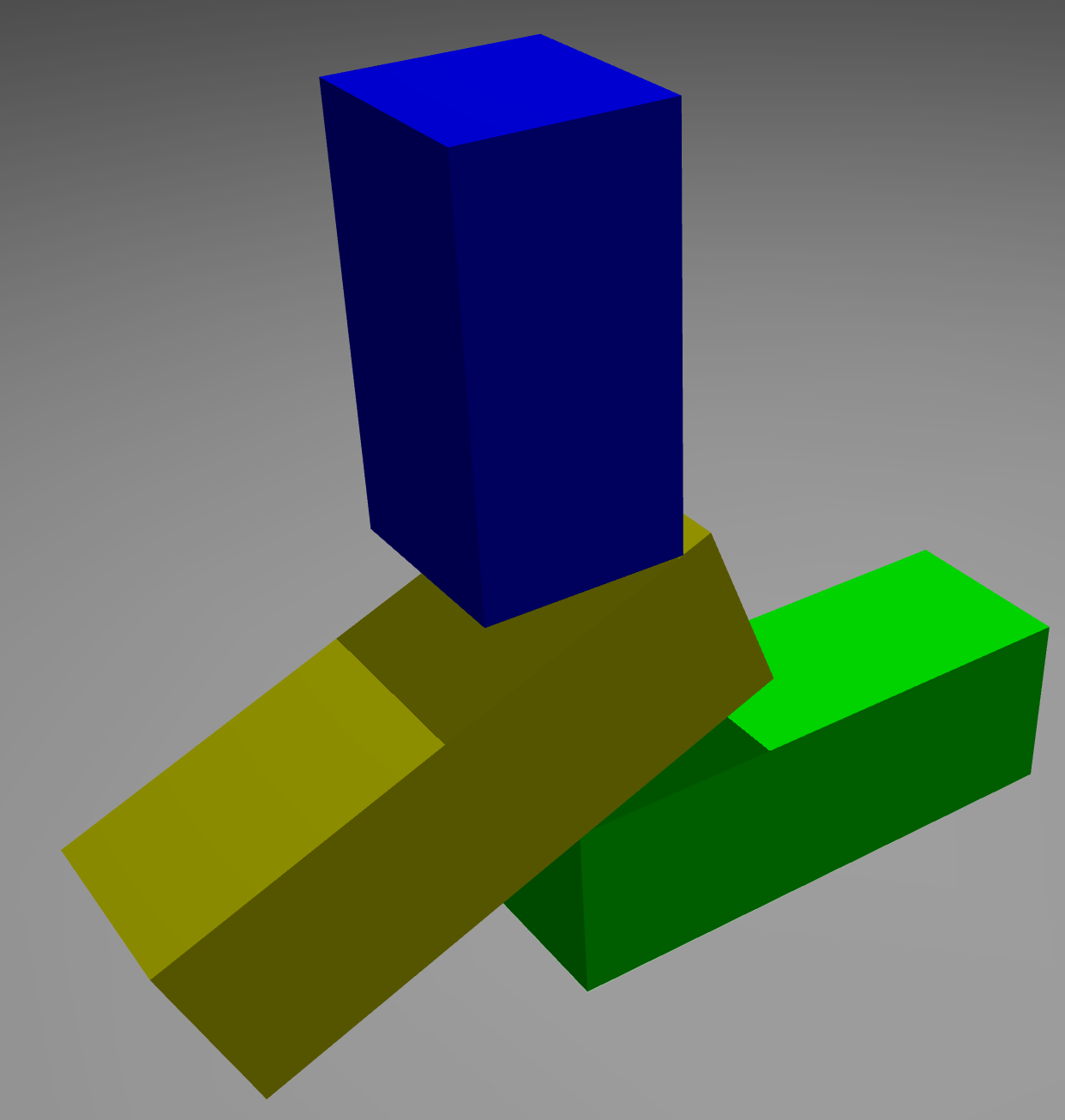}}
  \caption{The red arrow shows one detected qualitative force that has strict conditions to be instantiated. The real force that can cause this movement should have a large magnitude along the horizontal direction while slightly go downwards. %\footnotesize
  }
\end{figure}

At first glance, there are lots of possibilities, however, these possibilities can be further eliminated if we could infer an range of possible action force directions or locations. We also construct a scene that has a simple structure (see Fig.~\ref{fig:simple_example_before}-c) and manually verify each identified qualitative forces to see whether they can cause the observed movement as the actual force does. Surprisingly, most of the reported solutions can make it, and some solution (see Fig.~\ref{fig:simple_example_before}) requires very strict conditions to be successfully instantiated in the simulation.

%\paragraph{Real World}
In the real world experiment, we obtain visual inputs of a scene using Kinect 2 that generates RGBD images. The time gap between the before and after scenes is around 1 second. We use a segmentation algorithm provided in \cite{silberman2012indoor} to detect boxes and their spatial properties. We obtain the linear and angular velocity of each object from the differences in their positions and orientations at two different time points. To further demonstrate the capability of dealing with ambiguous information, we only use the contact points that are visible from the image. An example scene is shown in Fig.~\ref{fig:real_world_example}. Our method detected the real solution and found 80 possible qualitative forces on either the blue or red box.
\begin{figure}
\centering
  \subcaptionbox{\label{fig:real_world_example_before}}{\includegraphics[height=4cm, width=0.48\linewidth]{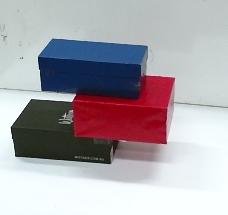}}
    \subcaptionbox{\label{fig:real_world_example_after}}{\includegraphics[height=4cm, width=0.48\linewidth]{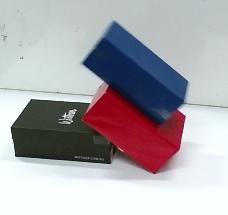}}
      \subcaptionbox{\label{fig:real_world_segmentation}}{\includegraphics[height=4cm]{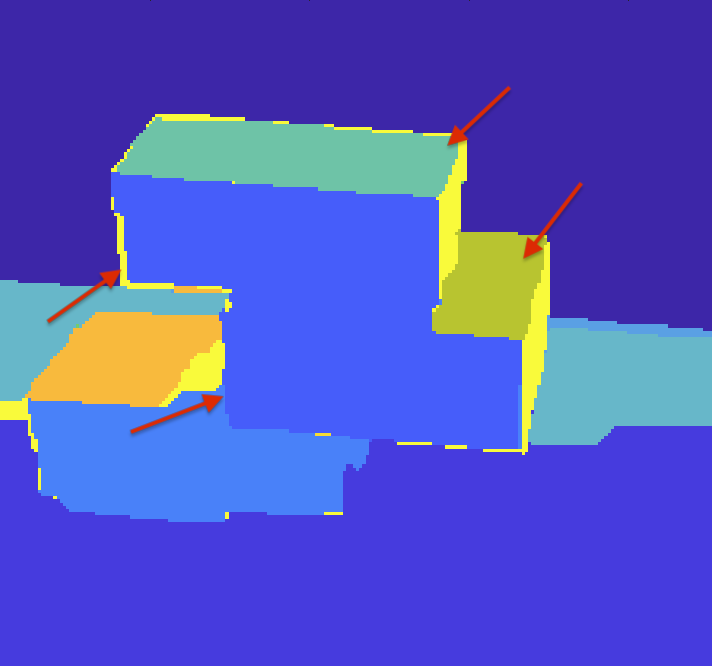}}
  \caption{\label{fig:real_world_example} (a-b) The before and after image of an real world example. (c) The segmentation of the before image. We indicate some identified solutions in red arrows %\footnotesize
  }
\end{figure}
\subsection{Discussion, Generalisation and Future Work}
The algorithm we present relies on very general assumptions. Given a specific problem domain, we can keep adding realistic assumptions about the domain for better performance. As hinted above, the algorithm can benefit from spatial reasoning algorithms that help to restrict the range of possible qualitative actions further. For example, in a robotics manipulation scenario, we could infer the range of the space that a robotic arm can reach and use this knowledge to prune those unreachable actions. This reasoning could be done either numerically such as trajectory planning or qualitatively such as reasoning about qualitative rotations of arms. The explanation generated by the algorithm is in the form of a set of the assigned qualitative forces which are readily be labelled with causal roles according to the theory proposed in \cite{wolff2015causal}.

Our theory can work with qualitative and quantitative constraints in the same way as we specified the rules. One future direction can be extending the theory to cover magnitude quantities. By this we could infer whether the mass or inertial tensor of an object would prevent any potential movement caused by a force of certain magnitude. Another promising direction is to use the qualitative inference to guide the search of solutions for simulation-based methods. The hierarchical optimisation framework proposed in \cite{tossaint2015logic} could be one possibility.

The completeness proof only holds when we know the qualitative locations, i.e., the tuple $\langle \fr, O\rangle$, of all the contact points beforehand. However, it could be possible that some of the contact points are missing due to imperfect perception or unobserved collisions between $t_1$ and $t_2$. To deal with this problem, we can extend the formulation of AIP-SAT by adding free variables representing the forces on unknown contact points. Each free variable has the same form as the action variable of which the qualitative location is not fixed. Hence, it would be desirable to have a method to tell if the qualitative location of each assignment is actually reachable by the involved objects. A simple heuristic can be to check whether there is an open path between the two collided objects given their motion state. As inferring potential collisions between multiple moving objects %is not the focus of this paper and the problem
itself is a challenging problem in both simulation (collision detection \cite{kockara2007collision}) and qualitative reasoning (\cite{ge2016hole}) area, we left the investigation of this problem as future work.

\section{Conclusion}
This paper proposed a qualitative theory for the motion of rigid objects based on the modelling approaches from qualitative reasoning and physics simulation areas. Based on the formulation we solved an interesting action inference problem. We proved the completeness of the algorithm and applied it to both simulated and real-world environments. We hope this work opens up a new direction of using qualitative reasoning for drawing explainable physical inferences in daily scenarios, which could contribute to achieving the goal of building intelligent physical systems.

\bibliographystyle{aaai}
\bibliography{ijcai18}

\end{document}